\documentclass[11pt,a4paper]{article}
\usepackage[hyperref]{acl2021}
\usepackage{times}
\usepackage{latexsym}
\usepackage{amssymb}
\usepackage{mathtools}
\usepackage{cleveref}
\usepackage{bm}
\usepackage{soul}
\usepackage{microtype}
\usepackage{multirow}

\usepackage{amssymb,amsthm,mathtools,bbm}
\newtheorem{theorem}{Theorem}[section]
\newtheorem{proposition}[theorem]{Proposition}

\newtheorem{definition}[theorem]{Definition}

\crefformat{section}{\S#2#1#3}
\crefformat{subsection}{\S#2#1#3}
\crefformat{subsubsection}{\S#2#1#3}

\usepackage[inline]{enumitem}

\DeclareMathOperator{\Span}{span}
\DeclareMathOperator{\Cs}{Cs}
\DeclareMathOperator{\Rs}{Rs}
\DeclareMathOperator{\LN}{LN}
\DeclareMathOperator{\Real}{\mathbb{R}}
\DeclareMathOperator{\softmax}{softmax}
\DeclareMathOperator{\A}{\textbf{A}}
\DeclareMathOperator{\Atilde}{\Tilde{\A}}
\DeclareMathOperator{\Q}{\textbf{Q}}
\DeclareMathOperator{\K}{\textbf{K}}
\DeclareMathOperator{\V}{\textbf{V}}
\DeclareMathOperator{\He}{\textbf{H}}
\DeclareMathOperator{\D}{\textbf{D}}
\DeclareMathOperator{\T}{\textbf{T}}
\DeclareMathOperator{\Pm}{\textbf{P}}
\DeclareMathOperator{\Qm}{\textbf{Q}}

\DeclareMathOperator{\ReLU}{ReLU}
\DeclareMathOperator{\Linear}{Linear}
\DeclareMathOperator{\Norm}{Norm}
\DeclareMathOperator{\Min}{min}
\DeclareMathOperator{\Rank}{rank}
\newcommand{\modelname}{Transformer}

\def\aclpaperid{3857} 

\title{A Deeper Analysis for Improving Identifiability in Transformers}

\title{A Deeper Analysis for Improving Identifiability in Transformers without\\ Compromising their Performance in Diverse Text Classification Tasks}

\title{Improving Identifiability in Transformers without\\ Compromising their Performance in Varied Text Classification Tasks}

\title{More Identifiable yet Equally Performant Transformers\\ for Text Classification}

\author{First Author \\
  Affiliation / Address line 1 \\
  Affiliation / Address line 2 \\
  Affiliation / Address line 3 \\
  \texttt{email@domain} \\\And
  Second Author \\
  Affiliation / Address line 1 \\
  Affiliation / Address line 2 \\
  Affiliation / Address line 3 \\
  \texttt{email@domain} \\}

\date{}
\author{Rishabh Bhardwaj$^1$, Navonil Majumder$^1$, Soujanya Poria$^1$, Eduard Hovy$^2$\\[1ex]
$^1$~Singapore University of Technology and Design, Singapore\\
$^2$~Carnegie Mellon University, Pittsburgh, PA, USA  \\
\texttt{rishabh\_bhardwaj@mymail.sutd.edu.sg}\\
\texttt{\{navonil\_majumder, sporia\}@sutd.edu.sg}\\ \texttt{hovy@cs.cmu.edu} 
}

\aclfinalcopy 
\begin{document}
\maketitle
\begin{abstract}
Interpretability is an important aspect of the trustworthiness of a model's predictions. Transformer's predictions are widely explained by the attention weights, i.e., a probability distribution generated at its self-attention unit (head). Current empirical studies provide shreds of evidence that attention weights are not explanations by proving that they are not unique. A recent study showed theoretical justifications to this observation by proving the non-identifiability of attention weights. For a given input to a head and its output, if the attention weights generated in it are unique, we call the weights identifiable. In this work, we provide deeper theoretical analysis and empirical observations on the identifiability of attention weights. Ignored in the previous works, we find the attention weights are more identifiable than we currently perceive by uncovering the hidden role of the key vector. However, the weights are still prone to be non-unique attentions that make them unfit for interpretation. To tackle this issue, we provide a variant of the encoder layer that decouples the relationship between key and value vector and provides identifiable weights up to the desired length of the input. We prove the applicability of such variations by providing empirical justifications on varied text classification tasks. The implementations are available at \url{https://github.com/declare-lab/identifiable-transformers}.
\end{abstract}

\section{Introduction}
Widely adopted \modelname{} architecture \cite{vaswani2017attention} has obviated the need for sequential processing of the input that is enforced in traditional Recurrent Neural Networks (RNN). As a result, compared to a single-layered LSTM or RNN model, a single-layered \modelname{} model is computationally more efficient, reflecting in a relatively shorter training time~\cite{vaswani2017attention}. This advantage encourages the training of deep \modelname-based language models on large-scale datasets. Their learning on large corpora has already attained state-of-the-art (SOTA) performances in many downstream Natural Language Processing (NLP) tasks. A large number of SOTA machine learning systems even beyond NLP \cite{lu2019vilbert} are inspired by the building blocks of \modelname{} that is multi-head self-attention \cite{radford2018improving, devlin2018bert}. 

A model employing an attention-based mechanism generates a probability distribution $\mathbf{a}$~=~$\{a_1, \ldots, a_n\}$ over the $n$ input units $z$~=~$\{z_1, \ldots, z_n\}$. The idea is to perform a weighted sum of inputs, denoted by $\sum_{i=1}^{n} a_iz_i$, to produce a more context-involved output. The attention vector, $\mathbf{a}$, are commonly interpreted as scores signifying the relative importance of input units. However, counter-intuitively, it is recently observed that the weights generated in the model do not provide meaningful explanations \cite{jain2019attention,wiegreffe2019attention}.

Attention weights are (structurally) \textit{identifiable} if we can uniquely determine them from the output of the attention unit \cite{brunner2019identifiability}. Identifiability of the attention weights is critical to the model's prediction to be interpretable and replicable.  If the weights are not unique, explanatory insights from them might be misleading.  

The \textit{self}-attention transforms an input sequence of vectors $z$~=~$\{z_1, \ldots, z_n\}$ to a contextualized output sequence $y$~=~$\{y_1, \ldots, y_n\}$, where $y_k$~=~$\sum_{i=1}^n a_{(k,i)}\;z_i$. The scalar $a_{(k,i)}$ captures how much of the $i_\text{th}$ token contributes to the contextualization of $k_\text{th}$ token. A \modelname{} layer consists of multiple heads, where each head performs self-attention computations, we break the head computations in two phases:

\begin{itemize}
    \item Phase 1: Calculation of attention weights $a_{(k,i)}$. It involves mapping input tokens to key and query vectors. The dot product of $k_\text{th}$ query vector and $i_\text{th}$ key vector gives $a_{(k,i)}$.
    \item Phase 2: Calculation of a contextualized representation for each token. It involves mapping input tokens to the value vectors. The contextualized representation for $k_\text{th}$ token can be computed by the weighted average of the value vectors, where the weight of $i_\text{th}$ token is $a_{(k,i)}$ computed in first phase.
\end{itemize}

The identifiability in \modelname{} has been recently studied by \citet{brunner2019identifiability} which provides theoretical claims that under mild conditions of input length, attention weights are not unique to the head's output. Essentially their proof was dedicated to the analysis of the computations in the second phase, i.e., token contextualization. However, the theoretical analysis ignored the crucial first phase where the attention weights are generated. Intrinsic to their analysis, the attention identifiability can be studied by studying only the second phase of head computations. However, even if we find another set of weights from the second phase, it depends on the first phase if those weights can be generated as the part of key-query multiplication. 

In this work, we probe the identifiability of attention weights in \modelname{} from a perspective that was ignored in \citet{brunner2019identifiability}. We explore the previously overlooked first phase of self-attention for its contribution to the identifiability in \modelname{}. During our analysis of the first phase, we uncover the critical constraint imposed by the size of the key vector\footnote{The size of key and query vector is expected to be the same due to the subsequent dot product operation} $d_k$. The flow of analysis can be described as
\begin{itemize}[itemsep=0em,wide,leftmargin=!]
    \item We first show that the attention weights are identifiable for the input sequence length $d_s$ no longer than the size of value vector $d_v$ (\cref{a_identi}) \cite{brunner2019identifiability}\footnote{The sequence length denotes number of tokens at input.}.
    \item For the case when $d_s > d_v$, we analyse the attention weights as raw dot-product (logits) and the $\softmax$ed dot-product (probability simplex), independently. An important theoretical finding is that both versions are prone to be unidentifiable.
    \item In the case of attention weights as logits (\cref{logits}), we analytically construct another set of attention weights to claim the unidentifiability. In the case of attention weights as $\softmax$ed logits (\cref{soft_logits}), we find the attention identifiability to be highly dependent on $d_k$. Thus, the size of key vector plays an important role in the identifiability of the self-attention head. The pieces of evidence suggest that the current analysis in \citet{brunner2019identifiability} ignored the crucial constraints from the first phase in their analysis.
\end{itemize}

To resolve the unidentifiability problem, we propose two simple solutions (\cref{sec:solutions}). For the regular setting of the \modelname{} encoder where $d_v$ depends on the number of attention heads and token embedding dimension, we propose to reduce $d_k$. This may lead to more identifiable attention weights. Alternatively, as a more concrete solution, we propose to set $d_v$ equal to token embedding dimension while adding head outputs as opposed to the regular approach of concatenation \cite{vaswani2017attention}. Embedding dimension can be tuned according to the sequence length up to which identifiability is desired. We evaluate the performance of the proposed variants on varied text classification tasks comprising of ten datasets (\cref{classification}).

In this paper, our goal is to provide concrete theoretical analysis, experimental observations, and possible simple solutions to identifiability of attention weights in \modelname{}. The idea behind identifiable variants of the \modelname{} is---the harder it is to obtain alternative attention weights, the likelier is they are identifiable, which is a desirable property of the architecture. Thus, our contribution are as follows:
\begin{itemize}[itemsep=0em,wide,leftmargin=!]
    \item We provide a concrete theoretical analysis of identifiability of attention weights which was missing in the previous work by \citet{brunner2019identifiability}.
    \item We provide \modelname{} variants that are identifiable and validate them empirically by analysing the numerical rank of the attention matrix generated in the self-attention head of the \modelname{} encoder. The variants have strong mathematical support and simple to adopt in the standard Transformer settings.
    \item We provide empirical evaluations on varied text classification tasks that show higher identifiability does not compromise with the task's performance.
\end{itemize}

\section{Background}
\subsection{Identifiability}
A general trend in machine learning research is to mathematically model the input-output relationship from a dataset. This is carried out by quantitatively estimating the set of model parameters that best fit the data. The approach warrants prior (to fitting) examination of the following aspects:

\begin{itemize}[leftmargin=*, itemsep=0.1em]
    \item The sufficiency of the informative data to the estimate model parameters, i.e., practical identifiability. Thus, the limitation comes from the dataset quality or quantity and may lead to ambiguous data interpretations \cite{raue2009structural}.
    
    \item The possibility that the structure of the model allows its parameters to be uniquely estimated, irrespective of the quality or quantity of the available data. This aspect is called structural identifiability. A model is said to be structurally \textit{unidentifiable} if a different set of parameters yield the same outcome. 
\end{itemize}

In this work, we focus on the structural identifiability \cite{bellman1970structural}. It is noteworthy that the goodness of the fit of a model on the data does not dictate its structural identifiability. Similar to \citet{brunner2019identifiability}, we focus our analysis on the identifiability of attention weights, which are not model parameters, yet demands meaningful interpretations and are crucial to the stability of representations learned by the model.  

\subsection{Transformer Encoder Layer} \label{encoder}
We base our analysis on the building block of \modelname{}, i.e., the encoder layer \citep{vaswani2017attention}. The layer has two sub-layers. First sub-layer performs the multi-head self-attention, and second is feed-forward network. Given a sequence of tokens $\{x_1, \ldots, x_{d_s}\}$, an embedding layer transforms it to a set of vector $\{z_1, \ldots, z_{d_s}\} \in \Real^{d_e}$, where $d_e$ denotes token embedding dimension. To this set, we add vectors encoding positional information of tokens $\{p_1, \ldots, p_{d_s}\} \in \Real^{d_e}$.

\paragraph{Multi-head Attention.} \label{MHA}
Input to a head of multi-head self-attention module is \textbf{W} $\in \Real^{d_s \times d_e}$, i.e., a sequence of $d_s$ tokens lying in a $d_e$-dimensional embedding space. Tokens are projected to $d_q$-size query, $d_k$-size key, and $d_v$-size value vectors using linear layers, resulting in the respective matrices - Query $\Q \in \Real^{d_s \times d_q}$, Key $\K \in \Real^{d_s \times d_k}$, and Value $\V \in \Real^{d_s \times d_v}$. The attention weights $\A \in \Real^{d_s \times d_s}$ can be computed by
\begin{equation} \label{at_A}
    \A = \softmax\left(\frac{\Q\K^T}{\sqrt{d_q}}\right).
\end{equation}
The $(i,j)_\textsuperscript{th}$ element of $\A$ shows how much of $i_{\text{th}}$ token is influenced by $j_{\text{th}}$ token. The output of a head $\He \in \Real^{d_s \times d_e}$ is given by
\begin{equation} \label{at_T}
    \He = \A \V \D = \A \T,
\end{equation}
where $\D \in \Real^{d_v \times d_e}$ is a linear layer and the matrix $\T \in \Real^{d_s \times d_e}$ denotes the operation $\V \D$. The $\Real^{d_s \times d_e}$ output of multi-head attention can be expressed as a summation over $\He$ obtained for each head\footnote{For simplicity, we have omitted head indices.}. The $i\textsubscript{th}$ row of multi-head output matrix corresponds to the $d_e$ dimensional contextualized representation of $i\textsubscript{th}$ input token. In the original work, \citet{vaswani2017attention}, the multi-head operation is described as the concatenation of $\A \V$ obtained from each head followed by a linear transformation $\D \in \Real^{d_e \times d_e}$. Both the explanations are associated with the same sequence of matrix operations as shown in \cref{fig:hadd}.

\begin{figure}[t] 
    \centering 
    \includegraphics[width=1.0\linewidth]{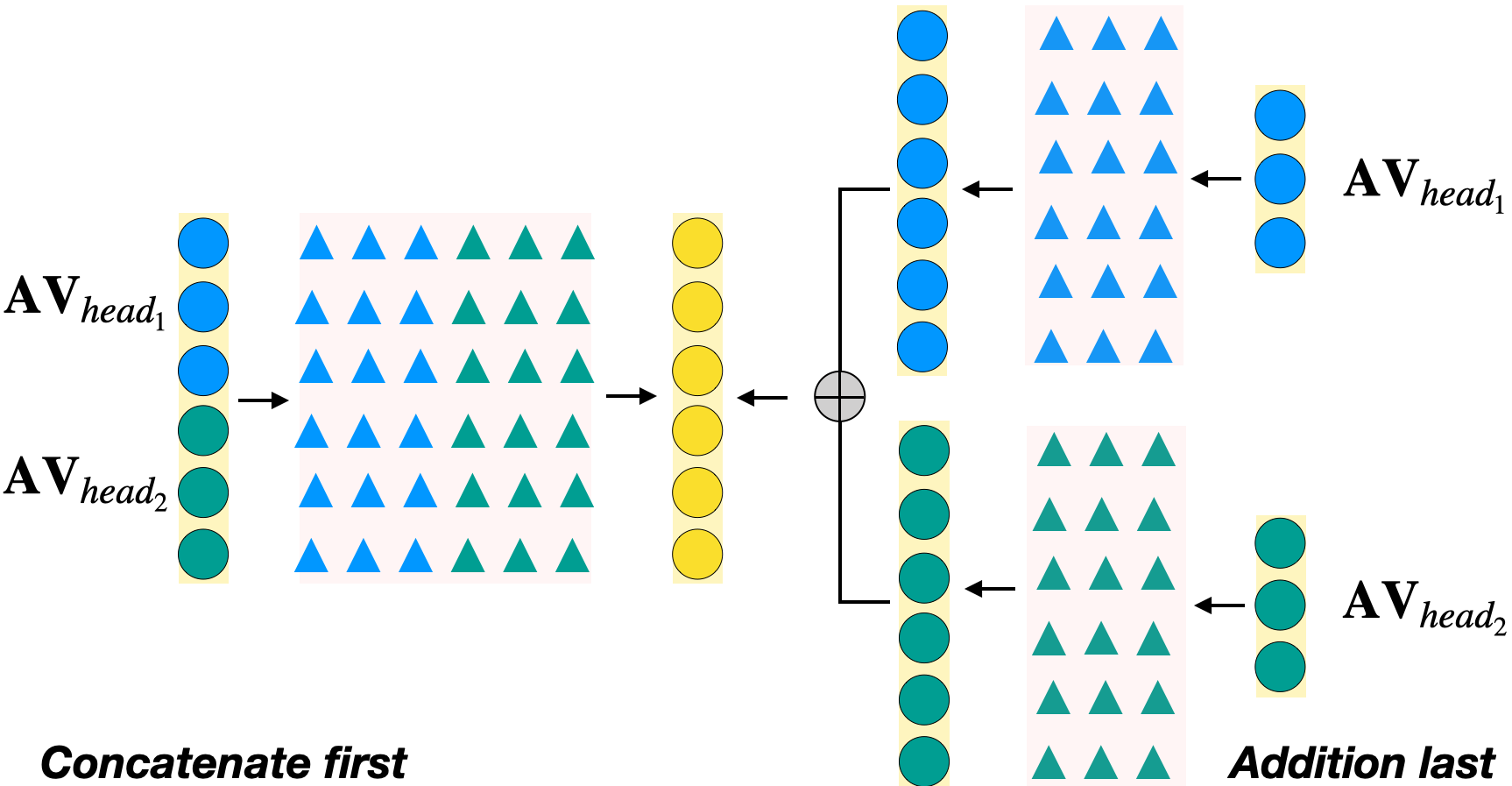} 
    \caption{\footnotesize An illustration for a \modelname{} with two-head attention units. Triangles depict matrix weights. The left side shows concatenation of head outputs fed to a linear layer. The right side shows another interpretation of the same set of operations where we consider a linear transform applied to each head first. The transformed head outputs are then added.}
    \label{fig:hadd}
\end{figure}

In regular \modelname{} setting, a token vector is $t_i \in \{(z_j+p_j)\}_{i=1}^{d_s}$ is $d_e$~=~$512$ dimensional, number of heads $h$=8, size of $d_k$=$d_q$=$d_v$=$d_e/h$=64.

\paragraph{Feed-Forward Network.}
This sub-layer performs the following transformations on each token representation at the output of a head:
\begin{align*}
    y_1 &= \Linear_1(\Norm(t_i+\text{head output for } t_i))\\ 
    y_2 &= \Norm(t_i+\ReLU(\Linear_2(y_1)))
\end{align*}
$\Linear_1$ and $\Linear_2$ are linear layers with 2048 and 512 nodes, respectively. $\Norm$ denotes mini-batch layer normalization.

\section{Identifiability of Attention} \label{identifiability}
The output of an attention head $\He$ is the product of $\A$ and $\T$ (\cref{at_T}). Formally, we define identifiability of attention in a head:

\begin{definition}
For an attention head's output $\He$, attention weights $\A$ are identifiable if there exists a unique solution of $\A\T=\He$.
\end{definition}

The above definition can be reformulated as

\begin{definition}
$\A$ is unidentifiable if there exist an $\Atilde$, $(\Atilde \neq \mathbf{0})$,  such that $(\A+\Atilde)$ is obtainable from phase-1 of head computations and satisfy
\begin{equation}\tag{constraint-R1}\label{atilda}
    (\A + \Atilde) \T=\A \T \implies \Atilde \T=\mathbf{0}.
\end{equation}
\end{definition}

Under this constraint, we get $\Tilde{a}_i \T=0$ where $\Tilde{a}_i$ is the $i_{th}$ row of $\Atilde$. The set of vectors which when multiplied to $\T$ gets mapped to zero describes the left null space of $\T$ denoted by $\LN(\T)$. The dimension of the left null space of $\T$ can be obtained by taking the difference of the total number of rows ($d_s$) and the number of linearly independent rows, i.e, rank of the matrix $\T$ denoted by $\Rank(\T)$. Let $\dim(\cdot)$ denotes the dimension of a vector space, then
\begin{align}
    \LN(\T)&=\{\mathbf{v} \mid \mathbf{v}^T\T=\mathbf{0}\}\\
    \dim\big(\LN(\T)\big)&=d_s-\Rank(\T) \label{dimLN}.
\end{align}

\subsection{``$\A$'' is Identifiable for $\mathbb{\textit{d}_\textit{s} \leq \textit{d}_\textit{v}}$} \label{a_identi}

If $\dim(\LN(\T))=0$ then $\LN(\T)=\{\mathbf{0}\}$, it leads to the only solution of \ref{atilda} that is $\Atilde=\mathbf{0}$. Therefore, the unidentifiabilty condition does not hold. Now we will prove such a situation exists when the number of tokens is not more than the size of value vector. 

The matrix $\T$ in \cref{at_T} is product of $d_s \times d_v$ value matrix $\V$ and $d_v \times d_e$ transformation $\D$. We utilize the fact that the rank of product of two matrices $\Pm$ and $\Qm$ is upper bounded by the minimum of $\Rank(\Pm)$ and $\Rank(\Qm)$, i.e., $\Rank(\Pm\Qm) \leq \min\big(\Rank(\Pm),\Rank(\Qm)\big)$. Thus, the upper bound on $\Rank(\T)$ in \cref{dimLN} can be determined by
\begin{equation} \label{rank_bound}
\small{%
\begin{aligned}
    \Rank(\T) &\leq \min \Big(\Rank(\V), \;\Rank(\D)\Big)\\
              &\leq \min \Big( \min(d_s, d_v), \min(d_v, d_e) \Big)\\
              &\leq \min \Big(d_s, d_v, d_v, d_e \Big)\\
              &\leq \min\Big(d_s, d_v\Big) \;\;\;\;\;\;\text{(as $d_e > d_v$)} \\
              &=\min\Big(d_s, 64\Big)
\end{aligned}
}%
\end{equation}
where the last inequality is obtained for a head in the regular \modelname{} for which $d_v$=64.

\paragraph{Numerical rank.}
To substantiate the bounds on $\Rank(\T)$ as derived above, we set up a model with a single encoder layer (\cref{exp_setup}). The model is trained to predict the sentiment of IMDB reviews (\cref{classification}). We feed the review tokens to the model and store the values generated in $\T$ of the first head. A standard technique for calculating the rank of a matrix with floating-point values and computations is to use singular value decomposition. The rank of the matrix will be computed as the number of singular values larger than the predefined threshold\footnote{The threshold value is $\max(d_s, d_e)*eps*||\T||_2$. The $eps$ is floating-point machine epsilon value, i.e., 1.19209e-07 in our experiments}. The \cref{fig:T_rank} illustrates how the rank changes with the sequence length $d_s$. The numerical rank provides experimental support to the theoretical analysis.
\begin{equation}
    \Rank(\T) =
\left\{
	\begin{array}{ll}
		d_s  & \mbox{if } d_s \leq d_v, \\
		d_v & \mbox{if } d_s > d_v.
	\end{array}
\right.
\end{equation}
Thus,
\begin{equation}\label{dim_null_proof}
    \begin{aligned}
            \dim\big(\LN(\T)\big) &= d_s - \Rank(\T) \\
            &=
        \left\{
        	\begin{array}{ll}
        		0  & \mbox{if } d_s \leq d_v, \\
        		(d_s - d_v) & \mbox{if } d_s > d_v.
        	\end{array}
        \right.\\
            &= \max\;(d_s - d_v, 0)
    \end{aligned}
\end{equation}
\begin{figure}[t] 
    \centering 
    \includegraphics[width=0.8\linewidth]{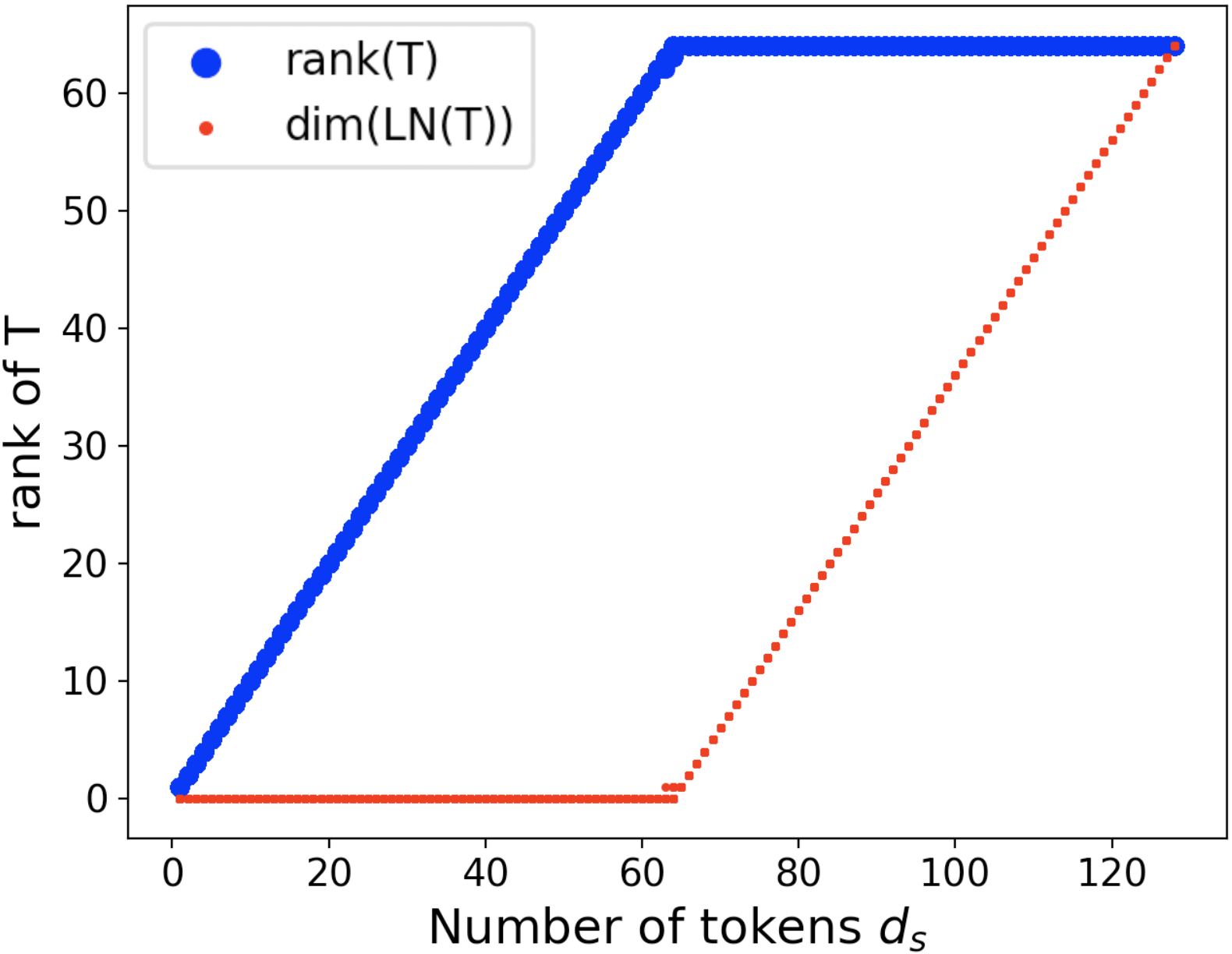} 
    \caption{\footnotesize Numerical rank of $\T$ (IMDB) and dimension of its left null space are scattered in blue and red, respectively.}
    \label{fig:T_rank}
\end{figure}
\noindent With this, we infer $\A$ is identifiable if $d_s \leq d_v$~=~$64$.
For the identifiability study, since we focus on a model's capability of learning unique attention weights, we will assume $\T$ has the maximum obtainable rank set by its upper bound.

\subsection{Idenitifiability when $\mathbf{\textit{d}_\textit{s} > \textit{d}_\textit{v}}$ \\ \hspace{8mm}(the hidden role of $\mathbf{\textit{d}_\textit{k}}$)} \label{a_non_identi}

In this case, from \cref{dim_null_proof}, we obtain a non zero value of $\dim\big(\LN(\T)\big)$. It allows us to find infinite $\Atilde$'s satisfying $(\A+\Atilde)\T=\A\T$. However, \ref{atilda} demands $\Atilde$ to be obtainable from the first phase of self-attention. As a first step, we focus our analysis on the attention matrix without applying $\softmax$ non-linearity, i.e., $\A = \left(\frac{\Q\K^T}{\sqrt{d_q}}\right)$. The analysis is crucial to identify constraints coming from the first phase of self-attention in \modelname{} that impact identifiability. Insights from this will help us analyse $\softmax$ version of $\A$.

\subsubsection{Attention Weights as Logits}\label{logits}
Since the logits matrix $\A$ is obtained from the product of $\Qm$ and $\K^T$, we can assert that
\begin{equation}
    \begin{aligned}
        \Rank(\A) &\leq \min\big(\Rank(\Qm), \Rank(\K^T)\big)\\
                  &\leq \min\big(d_e, d_k, d_q, d_e\big)\\
                  &= d_k.
    \end{aligned}
\end{equation}
Therefore, the rank of attention matrix producible by the head in the first phase of self-attention can at most be equal to the size of key vectors $d_k$. On this basis, the head can produce only those $\A + \Atilde$ satisfying
\begin{equation} \tag{constraint-R2}\label{constraint_0}
    \Rank(\A + \Atilde) \leq d_k
\end{equation}
\begin{proposition}
There exists a non-trivial $\Atilde$ that satisfy $(\A+\Atilde)\T=\A\T$ and \ref{constraint_0}. Hence, $\A$ is unidentifiable.
\end{proposition}
\begin{proof}
Let $a_1, \ldots, a_{d_s}$ and $\Tilde{a}_1, \ldots, \Tilde{a}_{d_s}$ denote rows of $\A$ and $\Atilde$, respectively. Without the loss of generality, let $a_1, \ldots, a_{d_k}$ be linearly independent rows. For all $j>d_k$, $a_j$ can be represented as a linear combination $\sum_{i=1}^{d_k}\lambda_{i}^ja_{i}$, where $\lambda_{i}^j$ is a scalar. Next, we independently choose first $k$ rows of $\Atilde$ that are $\{\Tilde{a}_1, \ldots, \Tilde{a}_{d_k}\}$ from $\LN(\T)$. From the same set of coefficients of linear combination $\lambda_i^j$ for $i \in \{1, \ldots, d_k\}$ and $j \in \{d_{k+1}, \ldots, d_s\}$, we can construct $j_\textsubscript{th}$ row of $\Atilde$ as $\Tilde{a}_j$~=~$\sum_{i=1}^{d_k}\lambda_{i}^j\Tilde{a_{i}}$. Now, since we can construct the $j_\text{th}$ row of $(\A+\Atilde)$ from the linear combination of its first $d_k$ rows~as~$\sum_{i=1}^{d_k}\lambda_{i}^j(a_{i}+\Tilde{a}_i)$, the rank of $(\A+\Atilde)$ is not more than~$d_k$. For a set of vectors lying in a linear space, a vector formed by their linear combination should also lie in the same space. Thus, the artificially constructed rows of $\Atilde$ belongs to $\LN(\T)$. Therefore, there exist an $\Atilde$ that establishes the proposition which claims the unidentifiability of $\A$.
\end{proof}

\subsubsection{Attention Weights as Softmaxed Logits} \label{soft_logits}
The $\softmax$ over attention logits generates attention weights with each row of $\A$ (i.e., $a_i$'s) is constrained to be a probability distribution. Hence, we can define constraint over $\Atilde$ as
\begin{align}
    (\A + \Atilde) &\geq \mathbf{0} \tag{P1} \label{constraint_1}\\
    \Atilde \T &= \mathbf{0} \tag{P2} \label{constraint_2}\\
    \Atilde \textbf{1} &= \mathbf{0} \tag{P3} \label{constraint_3}.
\end{align}
\ref{constraint_1} is non-negativity constraint on $(\A+\Atilde)$ as it is supposed to be the output of $\softmax$; \ref{constraint_2} denotes $\Atilde \in \LN(\T)$; \ref{constraint_3} can be derived from the fact $(\A + \Atilde)\textbf{1}$~=~\textbf{1} $\implies$ $(\A \textbf{1} + \Atilde \textbf{1})$~$=$~\textbf{1} $\implies \Atilde\textbf{1}$~$=$~$\mathbf{0}$ as ($\A\textbf{1}$~$=$~\textbf{1}). Where $\textbf{1} \in \Real^{d_s}$ is the vector of ones. The constraint in \ref{constraint_2} and \ref{constraint_3} can be combined and reformulated as $\Atilde[\T, \textbf{1}]=\mathbf{0}$. Following the similar analysis as in \cref{dim_null_proof}, we can obtain $\dim\big(\LN([\T, \textbf{1}])\big)=\max\big(d_s-(d_v+1),0\big)$. Disregarding the extreme cases when $a_i$ is a one-hot distribution, \citet{brunner2019identifiability} proved the existence and construction of non-trivial $\Atilde$'s satisfying all the constraints \ref{constraint_1}, \ref{constraint_2}, and \ref{constraint_3}.\footnote{For the sake of brevity, we skip the construction method.}

However, the proof by \citet{brunner2019identifiability} missed the \ref{constraint_0}, hence the existence of a non-trivial $\Atilde$ satisfying only the set of constraints \ref{constraint_1}, \ref{constraint_2} and \ref{constraint_3} may not be a valid proposition to claim attention weights unidentifiability. Essentially, the work largely ignored the constraints coming from the rank of the matrix that produces $\A$ after $\softmax$~\footnote{(input to the $\softmax$ is equivalent to $\A$ in \cref{logits})}. Let $\A_l$ denote logits $\left(\frac{\Q\K^T}{\sqrt{d_q}}\right)$ and $\softmax(\A_l)=(\A + \Atilde)$, where $\softmax$ is operated over each row of $\A_l$. We add an extra constraint on $\A_l$
\begin{equation}\tag{P4}\label{constraint_4}
    \Rank(\A_l) \leq d_k.
\end{equation}
The constraint \ref{constraint_4} confirms if there exists a logit matrix $\A_l$ that can generate $(\A+\Atilde)$, given constraints \ref{constraint_1}, \ref{constraint_2}, and \ref{constraint_3} are satisfied. The possibility of such an $\A_l$ will provide sufficient evidence that $\A$ is unidentifiable. Next, we investigate how the existence of $\Atilde$ is impacted by the size of key vector $d_k$ (query and key vector sizes are the same, i.e., $d_q$=$d_k$).

Let $(\A+\Atilde)(i,k)$ denotes $(i,k)\textsubscript{th}$ element of the matrix. We can retrieve the set of matrices $\A_l$ such that $\softmax(\A_l)={\A + \Atilde}$, where
\begin{equation}
    \label{a_l}
    \A_l(i,k) = c_i + \log(\A+\Atilde)(i,k)
\end{equation}
for some arbitrary $c_i \in \Real$; $\log$ denotes natural logarithm. As shown in \cref{fig:matrix_vector}, the column vectors of $\A_l$ can be written as $\bm{c}+\mathbf{\hat{a}_1}, \ldots, \bm{c}+\mathbf{\hat{a}_{d_s}}$. 

\begin{figure}[t] 
    \centering
    \includegraphics[width=0.9\linewidth]{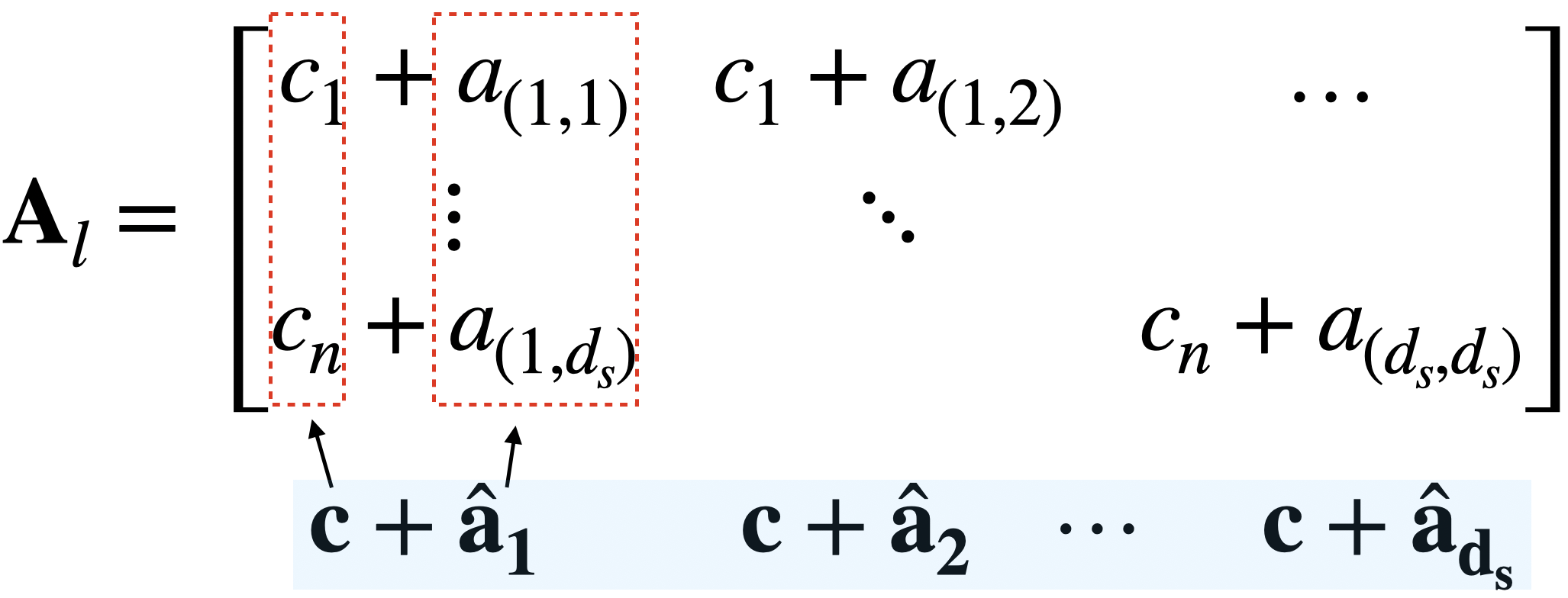} 
    \caption{\footnotesize Column vectors $(\bm{c}+\mathbf{\hat{a}_k})$ of $\A_l$, where $a_{(i,k)}$ represents $\log(\A+\Atilde)(i,k)$.}
    \label{fig:matrix_vector}
\end{figure}

For an arbitrarily picked $\Atilde$ satisfying constraint \ref{constraint_1}, \ref{constraint_2}, and \ref{constraint_3}, the dimensions of affine span $\mathcal{S}$ of $\{\mathbf{\hat{a}_1}, \ldots, \mathbf{\hat{a}_{d_s}}\}$ could be as high as $d_s-1$ (\cref{fig:affine_space}). In such cases, the best one could do is  to choose a $\mathbf{c_a} \in \mathcal{S}$ such that the dimension of the linear span of $\{\mathbf{\hat{a}_1 - c_a}, \ldots, \mathbf{\hat{a}_{d_s} - c_a}\}$, i.e., $\Rank(\A_l)$ is $d_s-1$. Hence, to satisfy \ref{constraint_4}, $d_s-1 \leq d_k \implies d_s \leq d_k + 1$. Thus, the set of $(\A+\Atilde)$ satisfying constraint \ref{constraint_1}, \ref{constraint_2} and \ref{constraint_3} are not always obtainable from attention head for $d_s > d_k$. We postulate

\begin{quote}
    Although it is easier to construct $\Atilde$ satisfying constraints \ref{constraint_1}, \ref{constraint_2} and \ref{constraint_3}, it is hard to construct $\Atilde$ satisfying constraint \ref{constraint_4} over the rank of logit matrix $\A_l$. Therefore, $\A$ becomes more identifiable as the size of key vector decreases.
\end{quote}

\begin{figure}[h] 
    \centering
    \includegraphics[width=0.6\linewidth]{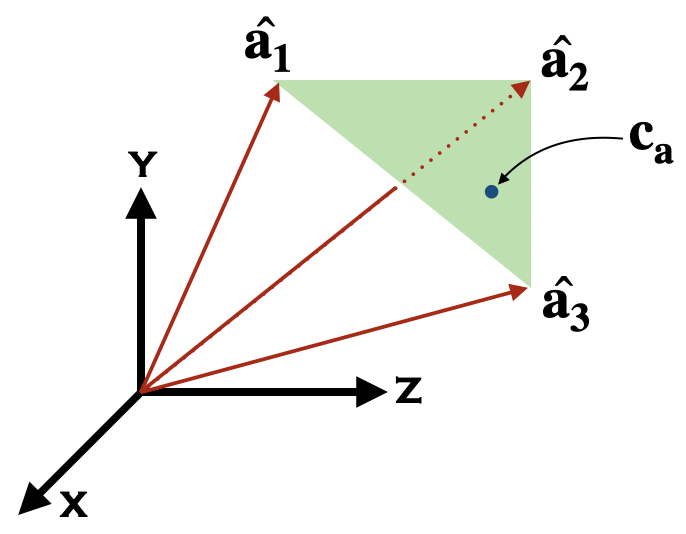} 
    \caption{This is a simplified illustration for the case $d_s=3$. Affine space (translated linear subspace) spanned by vectors $\mathbf{\hat{a}_1}$, $\mathbf{\hat{a}_2}$ and $\mathbf{\hat{a}_3}$. $\mathbf{c_a}$ can be any arbitrary vector in affine space. By putting $\mathbf{c}=-\mathbf{c_a}$, we can obtain a linear subspace whose rank is equal to rank of the affine subspace.}
    
    \label{fig:affine_space}
\end{figure}

\paragraph{Experimental evidence.} We conduct an experiment to validate the minimum possible numerical rank of $\A_l$ by constructing $\Atilde$. For $\Atilde$ to be obtainable from the phase 1, the minimum possible rank of $\A_l$ should not be higher than $d_k$. From IMDB dataset (\cref{classification}), we randomly sample a set of reviews with token sequence length $d_s$ ranging from 66 to 128 \footnote{$\dim\big(\LN(\T,\textbf{1})\big) > 0$ for $d_s>d_v+1=65$}. For each review, we construct 1000 $\Atilde$'s satisfying constraints \ref{constraint_1}, \ref{constraint_2}, and \ref{constraint_3} ---

First, we train a \modelname{} encoder-based IMDB review sentiment classifier (\cref{exp_setup}). We obtain an orthonormal basis for the left null space of $[\T,\textbf{1}]$ using singular value decomposition. To form an $\Atilde$, we generate $d_s$ random linear combinations of the basis vectors (one for each of its row). Each set of linear combination coefficients is sampled uniformly from $[-10, 10]$. All the rows are then scaled to satisfy the constraint \ref{constraint_1} as mentioned in \citet{brunner2019identifiability}. Using \cref{a_l}, we obtain a minimum rank matrix $\A_l$'s by putting $\textbf{c}=-\mathbf{\hat{a}_1}$. \Cref{fig:affine_space_rank} depicts the obtained numerical rank of $\A_l$. We observed all the obtained $\A_l$ from $(\A + \Atilde)$ (using \cref{a_l}) are full-row rank matrices. However, from the first phase of self-attention, the maximum obtainable rank of $\A_l$ is $d_k=64$. Thus, the experimentally constructed $\A_l$'s do not claim unidentifiability of $\A$ as it fails to satisfy the constraint \ref{constraint_4}, while for \citet{brunner2019identifiability}, it falls under the solution set to prove unidentifiability as it meets constraints \ref{constraint_1}, \ref{constraint_2}~and~\ref{constraint_3}. 

\begin{figure}[h] 
    \centering
    \includegraphics[width=0.9\linewidth]{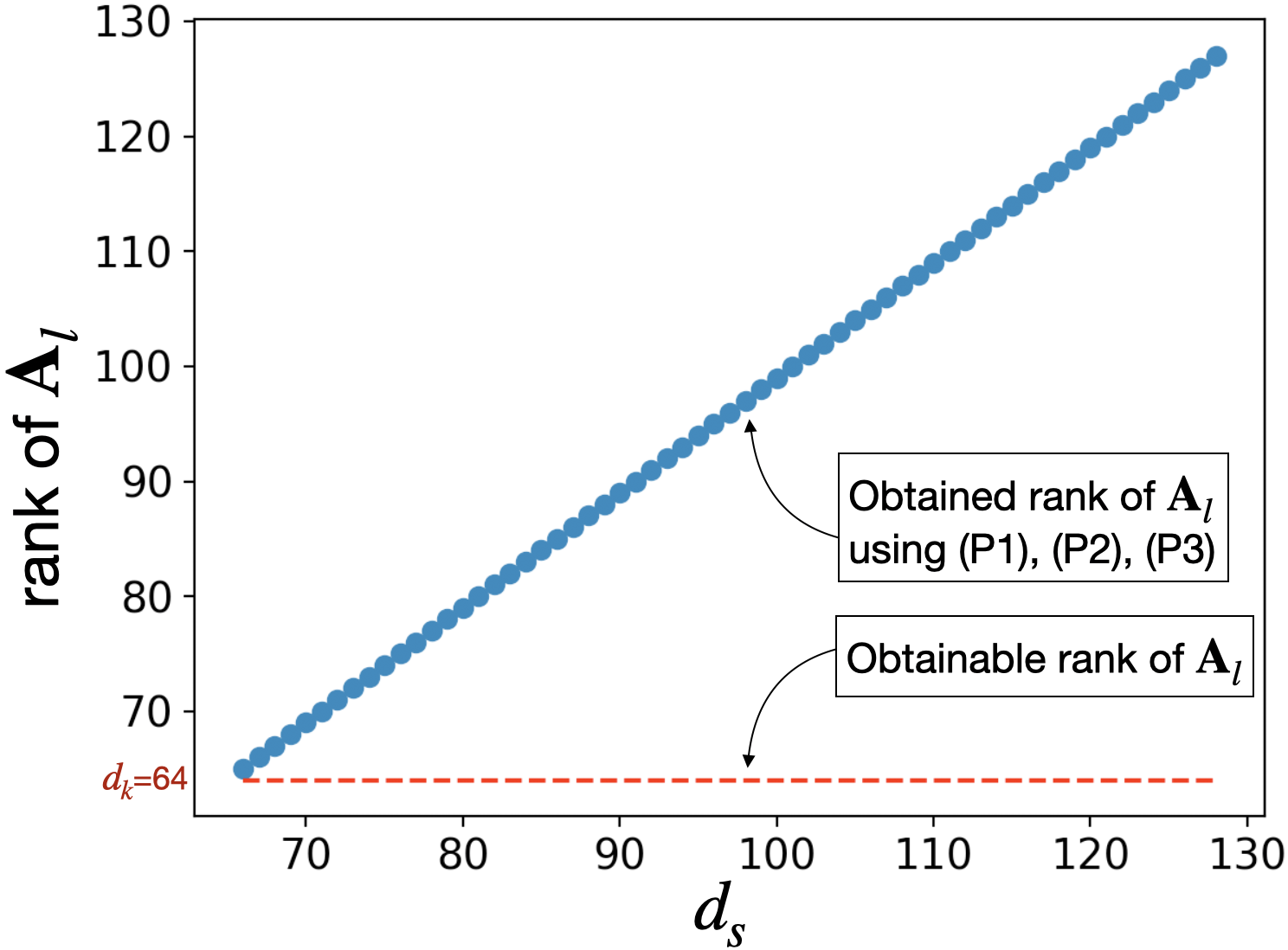} 
    \caption{\footnotesize The blue curve denotes the expected rank of $\A_l$'s obtained from $(\A+\Atilde)$, where $\Atilde$ satisfies the constraints \ref{constraint_1}, \ref{constraint_2}, and \ref{constraint_3}. The red curve denotes the maximum permissible rank of $\A_l$ that is obtainable from phase 1 of the head.}
    
    \label{fig:affine_space_rank}
\end{figure}

\section{Solutions to Identifiability} \label{sec:solutions}
Based on the Identifiability analysis in \cref{identifiability}, we propose basic solutions to make \modelname{}'s attention weights identifiable.

\paragraph{Decoupling $\mathbf{\textit{d}_\textit{k}}$.} Contrary to the regular \modelname{} setting where $d_k=d_v$, a simple approach is to decrease the value of $d_k$ that is the size of the key and query vector. It will reduce the possible solutions of $\Atilde$ by putting harder constraints on the rank of attention logits, i.e., $\A_l$ in \cref{a_l}. However, theoretically, $d_k$ decides the upper bound on dimensions of the space to which token embeddings are projected before the dot product. Higher the upper bound, more degree of freedom to choose the subspace dimensions as compared to the lower $d_k$ variants. Thus, there is a plausible trade-off when choosing between $d_k$ induced identifiability and the upper bound on the dimension of projected space.

\paragraph{Head Addition.}
To resolve the unidentifiability issue when sequence length exceeds the size of value vector, we propose to keep the value vector size and token embedding dimension to be more than (or equal to) the maximum allowed input tokens, i.e., $d_v \geq d_{s\text{-max}}$. In \citet{vaswani2017attention}, $d_v$ was bound to be equal to $d_e/h$, where $d_e$ is token embedding dimension and $h$ is number of heads. This constraint on $d_v$ is because of the concatenation of $h$ self-attention heads to produce $d_e$-sized output at the first sub-layer of the encoder. Thus, to decouple $d_v$ from this constraint, we keep $d_v=d_e$ and add each head's output.\footnote{$d_{s\text{-max}} < d_e$ as in the regular \modelname{} setting.}

\section{Classification Tasks}\label{classification}
For the empirical analysis of our proposed solutions as mentioned in \cref{sec:solutions}, we conduct our experiments on the following varied text classification tasks:

\subsection{Small Scale Datasets}
\paragraph{IMDB~\cite{IMDB}.} The dataset for the task of \textbf{sentiment classification} consist of IMDB movie reviews with their sentiment as positive or negative. Each of the train and test sets contain 25,000 data samples equally distributed in both the sentiment polarities.

\paragraph{TREC~\cite{trec_6}.} We use the 6-class version of the dataset for the task of \textbf{question classification} consisting of open-domain, facet-based questions. There are 5,452 and 500 samples for training and testing, respectively. 

\paragraph{SST~\cite{sst_data}.}Stanford sentiment analysis dataset consist of 11,855 sentences obtained from movie reviews. We use the 3-class version of the dataset for the task of \textbf{sentiment classification}. Each review is labeled as positive, neutral, or negative. The provided train/test/valid split is 8,544/2,210/1,101.

\subsection{Large Scale Datasets}
\paragraph{SNLI~\cite{bowman2015large}.} The dataset contain 549,367 samples in the training set, 9,842 samples in the validation set, and 9,824 samples in the test set. For the task of recognizing \textbf{textual entailment}, each sample consists of a premise-hypothesis sentence pair and a label indicating whether the hypothesis entails the premise, contradicts it, or neutral.

Please refer to \citet{zhang2015character} for more details about the following datasets:

\paragraph{Yelp.} We use the large-scale Yelp review dataset for the task of binary \textbf{sentiment classification}. There are 560,000 samples for training and 38,000 samples for testing, equally split into positive and negative polarities.

\paragraph{DBPedia.} The Ontology dataset for \textbf{topic classification} consist of 14 non-overlapping classes each with 40,000 samples for training and 5,000 samples for testing.

\paragraph{Sogou News.} The dataset for \textbf{news article classification} consist of 450,000 samples for training and 60,000 for testing. Each article is labeled in one of the 5 news categories. The dataset is perfectly balanced.

\paragraph{AG News.} The dataset for the \textbf{news articles classification} partitioned into four categories. The balanced train and test set consist of 120,000 and 7,600 samples, respectively.

\paragraph{Yahoo! Answers.} The balanced dataset for 10-class \textbf{topic classification} contain 1,400,000 samples for training and 50,000 samples for testing.

\paragraph{Amazon Reviews.} For the task of \textbf{sentiment classification}, the dataset contain 3,600,000 samples for training and 400,000 samples for testing. The samples are equally divided into positive and negative sentiment labels.

Except for the SST and SNLI, where the validation split is already provided, we flag 30\% of the train set as part of the validation set and the rest 70\% were used for model parameter learning.

\section{Experimental Setup} \label{exp_setup}
\paragraph{Setting up the encoder.} We normalize the text by lower casing, removing special characters, etc.\footnote{\url{https://pytorch.org/text/_modules/torchtext/data/utils.html}} For each task, we construct separate 1-Gram vocabulary ($U$) and initialize a trainable randomly sampled token embedding ($U \times d_e$) from $\mathcal{N}(0,1)$. Similarly, we randomly initialize a ($d_{\text{\textit{s}-max}} \times d_e$) positional embedding.

The encoder (\cref{encoder}) takes input a sequence of token vectors ($d_s \times d_e$) with added positional vectors. The input is then projected to key and query vector of size $d_k \in \{1,2,4,8,16,32,64,128,256\}$. For the \textbf{regula}r \modelname{} setting, we fix the number of heads $h$ to 8 and the size of value vector $d_v=d_e/h$ that is 64. For each token at the input, the outputs of attention heads are concatenated to generate a $d_e$-sized vector. For the \textbf{identifiable} variant of the \modelname{} encoder, $d_v=d_e=512$, this is equal to $d_{\text{\textit{s}-max}}$ to keep it identifiable up to the maximum permissible number of tokens. The outputs of all the heads are then added. Each token's contextualized representations (added head outputs) are then passed through the feed-forward network (\cref{encoder}). For classification, we use the encoder layer's output for the first token and pass it through a linear classification layer. In datasets with more than two classes, the classifier output is $\softmax$ed. In the case of SNLI, we use the shared encoder for both premise and hypothesis; the output of their first tokens is then concatenated just before the final classification layer. We use Adam optimizer, with learning rate =0.001, to minimize the cross-entropy loss between the target and predicted~label. For all the experiments, we keep the batch size as 256 and train for 20 epochs. We report the test accuracy obtained at the epoch with the best validation accuracy.

\paragraph{Numerical rank.} To generate the numerical rank plot on IMDB dataset as shown in \cref{fig:T_rank}, we train a separate \modelname{} encoder-based classifier. For a particular $d_s$ value, we sample 100 reviews from the dataset with token length $\geq d_s$ and clip each review to the maximum length $d_s$. The clipping will ensure the number of tokens is $d_s$ before feeding it to the encoder. The numerical rank is calculated for $\mathbf{T}$'s obtained from the first head of the encoder.

\section{Results and Discussion}
\begin{table*}[ht]
\centering
\resizebox{0.8\linewidth}{!}{
\begin{tabular}{l|c|ccccccccc}
\hline
\multirow{2}{*}{\textbf{Dataset}} &\multirow{2}{*}{\textbf{Version}} &\multicolumn{9}{c}{\textbf{Size of key vector $(d_k)$}}\\

&&\textbf{1} &\textbf{2} &\textbf{4} &\textbf{8} &\textbf{16} &\textbf{32} &\textbf{64} &\textbf{128} &\textbf{256}\\\hline


\multirow{2}{*}{\textbf{IMDB}} &\textit{Con} &0.884 &0.888 &0.886 &0.888 &0.846 &0.824 &0.803 &0.788 &0.755 \\ 
  &\textit{Add} &0.888 &0.885 &0.887 &0.884 &0.886 &0.882 &0.877 &0.832 &0.825 \\ \cline{2-11}

\multirow{2}{*}{\textbf{TREC}} &\textit{Con} &0.836 &0.836 &0.840 &0.822 &0.823 &0.764 &0.786 &0.706 &0.737 \\  
  &\textit{Add} &0.841 &0.842 &0.835 &0.842 &0.841 &0.836 &0.809 &0.809 &0.771 \\ \cline{2-11}

\multirow{2}{*}{\textbf{SST}} &\textit{Con} &0.643 &0.625 &0.627 &0.609 &0.603 &0.582 &0.574 &0.573 &0.554 \\  
  &\textit{Add} &0.599 &0.618 &0.628 &0.633 &0.628 &0.629 &0.592 &0.581 &0.586 \\ \cline{2-11}

\multirow{2}{*}{\textbf{SNLI}} &\textit{Con} &0.675 &0.674 &0.673 &0.672 &0.662 &0.659 &0.659 &0.655 &0.648 \\ 
  &\textit{Add} &0.683 &0.677 &0.674 &0.676 &0.673 &0.669 &0.663 &0.664 &0.655 \\ \cline{2-11}

\multirow{2}{*}{\textbf{Yelp}} &\textit{Con} &0.913 &0.911 &0.907 &0.898 &0.879 &0.862 &0.857 &0.849 &0.837 \\ 
  &\textit{Add} &0.914 &0.915 &0.916 &0.914 &0.915 &0.916 &0.910 &0.909 &0.891 \\ \cline{2-11}

\multirow{2}{*}{\textbf{DBPedia}} &\textit{Con} &0.979 &0.977 &0.977 &0.971 &0.966 &0.961 &0.957 &0.951 &0.949 \\ 
  &\textit{Add} &0.979 &0.978 &0.979 &0.977 &0.978 &0.973 &0.970 &0.969 &0.964 \\ \cline{2-11}

\multirow{2}{*}{\textbf{Sogou}} &\textit{Con} &0.915 &0.907 &0.898 &0.900 &0.893 &0.888 &0.868 &0.858 &0.838 \\ 
  &\textit{Add} &0.915 &0.908 &0.906 &0.904 &0.913 &0.914 &0.910 &0.906 &0.899 \\ \cline{2-11}

\multirow{2}{*}{\textbf{AG News}} &\textit{Con} &0.906 &0.903 &0.904 &0.904 &0.886 &0.877 &0.870 &0.870 &0.869 \\ 
  &\textit{Add} &0.902 &0.908 &0.907 &0.906 &0.897 &0.899 &0.901 &0.897 &0.893 \\ \cline{2-11}

\multirow{2}{*}{\textbf{Yahoo}} &\textit{Con} &0.695 &0.690 &0.684 &0.664 &0.644 &0.627 &0.616 &0.597 &0.574 \\  
  &\textit{Add} &0.697 &0.695 &0.696 &0.693 &0.693 &0.694 &0.688 &0.649 &0.683 \\ \cline{2-11}

\multirow{2}{*}{\textbf{Amazon}} &\textit{Con} &0.924 &0.925 &0.923 &0.922 &0.900 &0.892 &0.887 &0.882 &0.873 \\ 
  &\textit{Add} &0.925 &0.923 &0.925 &0.924 &0.924 &0.920 &0.907 &0.896 &0.889 \\ \hline

\end{tabular}
}
\caption{\label{performance_table} The test accuracy on varied text classification tasks spread over ten datasets. \textit{Con} means the regular concatenation of heads with $d_v=d_e/h$, \textit{Add} denotes encoder variant where $d_v=d_e$ and outputs of heads are added. In the regular \modelname{} encoder $Con$, the concatenation of $d_v$-sized output of $h$ heads followed by $d_e \times d_e$ linear transformation can be understood as first doing linear $d_v \times d_e$ linear transform of each head and then addition of the transformed output (\cref{fig:hadd}). In the \textit{Add} variant, we first add $h$ $d_v$-sized head outputs followed by $d_e \times d_e$ linear transformation.}
\end{table*}

For the identifiable variant, similar to \cref{a_identi}, we plot the numerical rank of $\T$ with input sequence length as shown in \cref{fig:id_T_rank}. Unlike \cref{fig:T_rank}, where $\dim\big(\LN(\T)\big)$ linearly increases after $d_s=64$, we find the dimension is zero for a larger $d_s$ ($\sim 380$). The zero dimensional (left) null space of $\T$ confirms there exist no nontrivial solution to the constraint \ref{constraint_0}, i.e., $\Atilde=\{\mathbf{0}\}$. Thus, the attention weights $\A$ are identifiable for a larger range of length of the input sequence.

\begin{figure}[ht] 
    \centering 
    \includegraphics[width=0.9\linewidth]{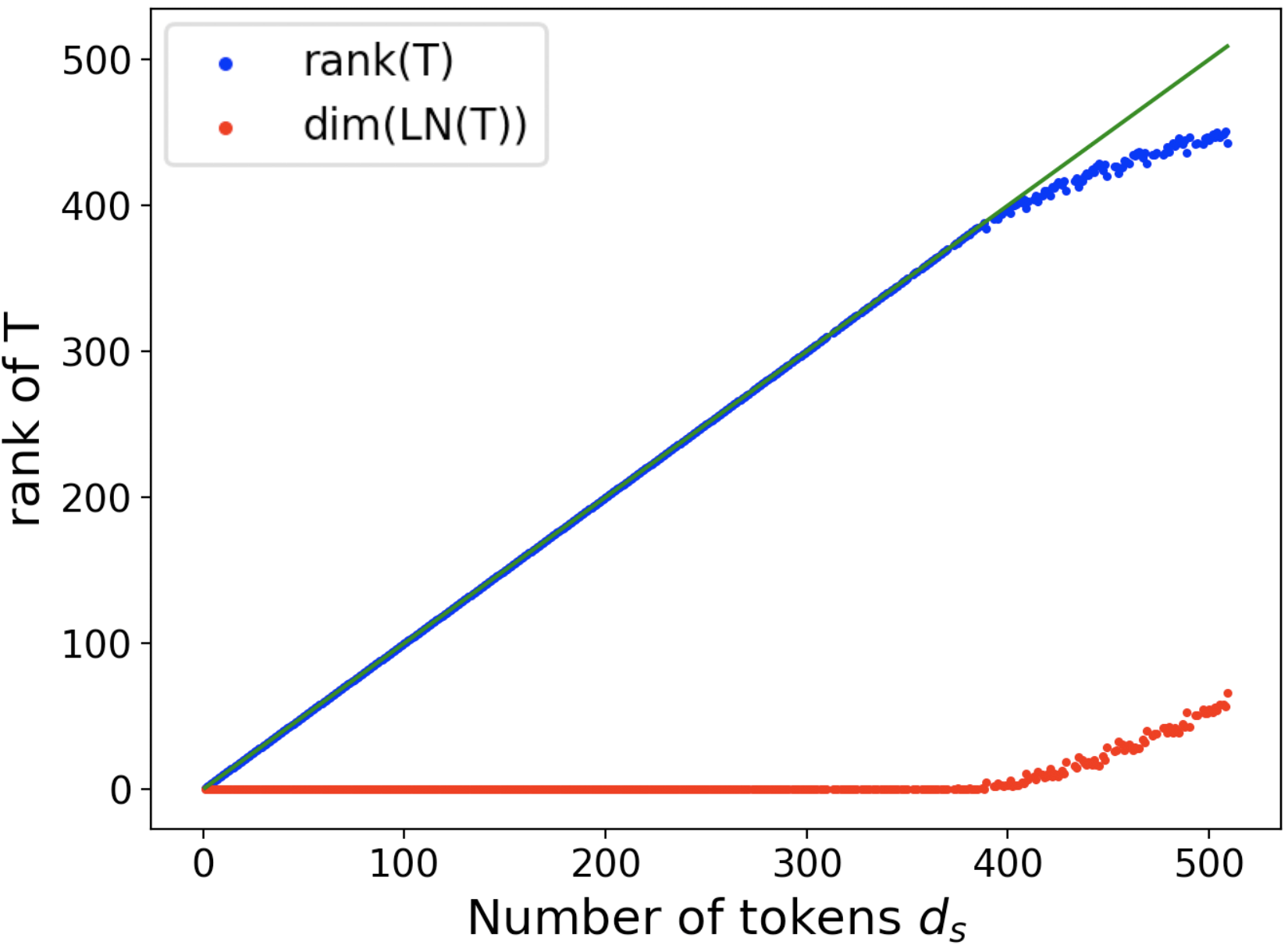} 
    \caption{Scatter plots in red and blue show $\Rank(\T)$ and $\dim\big(\LN(\T)\big)$, respectively, for matrices $\T$ obtained from the second phase of attention by feeding IMDB samples to the encoder. The green line shows the desired $\Rank(\T)$ for which $\dim\big(\LN(\T)\big)=0$ and thus attention weights are identifiable.} 
    \label{fig:id_T_rank}
\end{figure}

It is important that the identifiability of attention weights should not come at the cost of reduced performance of the model. To investigate this issue, we compare the performance of the identifiable \modelname{} encoder against its regular settings (\cref{exp_setup}) on varied text classification tasks.

For the regular setting, as discussed in \cref{sec:solutions} as one of the solutions, the \modelname{} can be made identifiable by decreasing the size of the key vector $d_k$. The rows of the \Cref{performance_table} corresponding to \textit{Con} denotes regular \modelname{} setting with varying size of key vector. We observe the classification accuracy at the lower $d_k$ is comparable or higher than large $d_k$ values, thus, the enhanced identifiability does not compromise with the model's classification accuracy. However, we notice a general performance decline with an increase in the size of the key vector. We speculate that for simple classification tasks, the lower-dimensional projection for key and query vector works well. However, as the task becomes more involved, a higher dimension for the projected subspace could be essential. Nonetheless, as we do not have strong theoretical findings, we leave this observation for future work.

Another solution to identifiability is to increase $d_v$ to $d_e$ and add the heads' outputs. This setting corresponds to the \textit{Add} rows in the \Cref{performance_table}. For key vector size $d_k$= 1, 2, and 4, We find the identifiable \modelname{}'s performance is comparable to the regular settings. For $d_k \geq 8$, as a general observation, we find the performance of \textit{Add} does not drop as drastically as  \textit{Con} with an increase in $d_k$. This could be due to the larger size of value vector leading to the more number of parameters in \textit{Add} that compensate for the significant reduction in the model's accuracy.

On the large-scale datasets, we observe that \textit{Add} performs slightly better than \textit{Con}. Intuitively, as shown in \cref{fig:hadd}, we can increase the size of value vector to increase the dimension of the space on which each token is projected. A higher dimensional subspace can contain more semantic information to perform the specific task.

Even though the theoretical analysis shows the possibility of a full row rank of $\T$ and identifiable attention weights, the $\T$ obtained from a trained model might not contain all the rows linearly independent as $d_s$ increases. We can explain this from the semantic similarities between words co-occurring together \cite{harris1954distributional}. The similarity is captured as the semantic relationship, such as dot product, between vectors in a linear space. As the number of tokens in a sentence, i.e., $d_s$ increases, it becomes more likely to obtain a token vector from the linear combination of other tokens.

\section{Conclusion}
\label{sec:conclusion}
This work probed \modelname{} for identifiability of self-attention, i.e., the attention weights can be uniquely identified from the head's output. With theoretical analysis and supporting empirical evidence, we were able to identify the limitations of the existing study by \citet{brunner2019identifiability}. We found the study largely ignored the constraint coming from the first phase of self-attention in the encoder, i.e., the size of the key vector. Later, we proved how we can utilize $d_k$ to make the attention weights more identifiable. To give a more concrete solution, we propose encoder variants that are more identifiable, theoretically as well as experimentally, for a large range of input sequence lengths. The identifiable variants do not show any performance drop when experiments are done on varied text classification tasks.
Future works may analyse the critical impact of identifiability on the explainability and interpretability of the \modelname{}. 


\section*{Acknowledgments}
This research is supported by A*STAR under its RIE 2020 Advanced Manufacturing and Engineering programmatic grant, Award No.– A19E2b0098.


\bibliographystyle{acl_natbib}
\bibliography{acl2021}
\clearpage

\appendix
\section{Background on Matrices}

\subsection{Span, Column space and Row space}
Given a set of vectors \textbf{V} $\coloneqq \{\text{\textbf{v}}_1, \text{\textbf{v}}_2, \ldots, \text{\textbf{v}}_n \}$, the span of \textbf{V},  $\Span(\text{\textbf{V}})$, is defined as the set obtained from all the possible linear combination of vectors in \textbf{V}, i.e., 
\begin{equation*}
\Span (\text{\textbf{V}})\coloneqq \{\sum_{i=1}^{n} \lambda_i \text{\textbf{v}}_i \mid \lambda_i \in \mathbb{R}, \;i \in \{1,2,\ldots,n\}\}.
\end{equation*}
The $\Span$(\textbf{V}) can also be seen as the smallest vector space that contains the set \textbf{V}.

Given a matrix $\A$ $\in \mathbb{R}^{m\times n}$, the column space of $\A$, $\Cs(\A)$, is defined as space spanned by its column vectors. Similarly, the row space of $\A$, $\Rs(\A)$, is the space spanned by the row vectors of $\A$. $\Cs(\A)$ and $\Rs(\A)$ are the subspaces of the real spaces $\Real^m$ and $\Real^n$, respectively. If the row vectors of $\A$ are linearly independent, the $\Rs(\A)$ will span $\Real^m$. A similar argument holds between $\Cs(\A)$ and $\Real^n$. 

\subsection{Matrix Rank}\label{rank}
The rank of a matrix $\Pm$ (denoted as $\Rank(\Pm))$ tells about the dimensions of the space spanned by the row vectors or column vectors. It can also be seen as the number of linearly independent rows or columns. The following properties hold
\begin{align*}
    \Rank\big(\Pm\big) &\leq \Min \Big(m_p, n_p\Big)\\
    \Rank\big(\Pm \Qm\big) &\leq \Min\Big(\Rank(\Pm), \Rank(\Qm)\Big).\\
\end{align*}
Where, $\Pm$ and $\Qm$ are $m_p \times n_p$ and $m_q \times n_q$ dimensional matrices, respectively.

\subsection{Null Space}\label{null}
The left null space of a $m_p \times n_p$ matrix $\Pm$ can be defined as the set of vectors \textbf{v} - 
\begin{equation}
    \LN\big(\Pm\big) = \{\text{\textbf{v}}^T\  \in \Real^{1\times m_p}\mid \text{\textbf{v}}^T \Pm = 0\}
\end{equation}

If the rows of $\Pm$ are linearly independent ($\Pm$ is full-row rank) the left null space of $\Pm$ is zero dimensional. The only solution to the system of equations $\text{\textbf{v}}\Pm = 0$ is trivial, i.e., \textbf{v}=0. The dimensions of the null space, known as nullity, of $\Pm$ can be calculated as
\begin{equation} \label{null_space}
    \dim \big(\LN(\Pm)\big) = m_p - \Rank(\Pm).
\end{equation}
The nullity of $\Pm$ sets the dimensions of the space \textbf{v} lies in. In \cref{identifiability}, we utilize our knowledge of \cref{rank} and \cref{null} to analyse identifiability in a \modelname{}.



\end{document}